\documentclass[onecolumn]{article}
\usepackage{times}
\usepackage{bm}
\usepackage{algorithm}
\usepackage{algorithmic}
\usepackage{amssymb}
\usepackage[cmex10]{amsmath}
\usepackage{amsthm}
\usepackage{amsfonts}
\usepackage{helvet}
\usepackage{graphicx}
\usepackage{caption}
\usepackage{courier}
\usepackage[margin=30mm]{geometry}

\newtheorem{thm}{Theorem}
\newtheorem{rem}{Remark}
\newtheorem{col}{Corollary}
\newtheorem{mydef}{Definition}
\newtheorem{lemma}{Lemma}

\begin{document}
%
\title{Robust Elastic Net Regression}
\author{Weiyang Liu\\
Peking Univ.\\
{\tt\small wyliu@pku.edu.cn}
\and
Rongmei Lin\\
Peking Univ.\\
{\tt\small scutrmlin@gmail.com}
\and
Meng Yang\\
Shenzhen Univ.\\
{\tt\small yang.meng@szu.edu.cn}
}
\maketitle
\begin{abstract}
\begin{quote}
We propose a robust elastic net (REN) model for high-dimensional sparse regression and give its performance guarantees (both the statistical error bound and the optimization bound). A simple idea of trimming the inner product is applied to the elastic net model. Specifically, we robustify the covariance matrix by trimming the inner product based on the intuition that the trimmed inner product can not be significantly affected by a bounded number of arbitrarily corrupted points (outliers). The REN model can also derive two interesting special cases: robust Lasso and robust soft thresholding. Comprehensive experimental results show that the robustness of the proposed model consistently outperforms the original elastic net and matches the performance guarantees nicely.
\end{quote}
\end{abstract}
\section{Introduction}
Over the past decades, sparse linear regression has been and is still one of the most powerful tools in statistics and machine learning. It seeks to represent a response variable as a sparse linear combination of covariates. Recently, sparse regression has received much attention and found interesting applications in cases where the number of variables $p$ is far greater than the number of observations $n$. The regressor in high-dimensional regime tends to be sparse or near sparse, which guarantees the high-dimensional signal can be efficiently recovered despite the underdetermined nature of the problem \cite{candes2005decoding,donoho2006compressed}. However, data corruption is very common in high-dimensional big data. Research has demonstrated the current sparse linear regression (e.g. Lasso) performs poorly when handling dirty data \cite{chen2013robust}. Therefore how to robustify the sparse linear regression becomes a major concern that draws increasingly more attentions.
\par
Robust sparse linear regression can be roughly categorized into several lines of researches. One type of representative approaches is to first remove the detected outliers and then perform the regression. However, outlier removal is not suitable for the high-dimensional regime, because outliers might not exhibit any strangeness in the ambient space due to the high-dimensional noise \cite{xu2013outlier}. Another type of approaches include replacing the standard mean squared loss with a more robust loss function such as trimmed loss and median squared loss. Such approaches usually can not give performance guarantees. Methods \cite{laska2009exact,nguyen2013exact,li2013compressed} have been developed to handle arbitrary corruption in the response variable, but fail with corrupted covariates. \cite{loh2012high} proposes a robust Lasso that considers the stochastic noise or small bounded noise in the covariates. \cite{chen2013noisy} also considers similar corruption settings and proposes a robust OMP algorithm for Lasso. For the same noise, \cite{rosenbaum2010sparse} proposes the matrix uncertainty selector that serves as a robust estimator.
\par
Other than the above noise, the persistent noise \cite{chen2013robust,feng2014robust} which is ill-modeled by stochastic distribution widely exists in many practical applications. We consider this type of noise in the paper. We make no assumptions on the distribution, magnitude or any other such properties of the noise (outliers). The only prior that we use is the bound on the fraction of corrupted samples. Aiming to address such noise problem in sparse linear regression, \cite{chen2013robust} proposes a robust Lasso that utilizes trimmed statistic. \cite{feng2014robust} also considers the same noise in the logistic regression using trimmed statistic. Yet robust Lasso still has some inevitable limitations that are inherent from the Lasso \cite{tibshirani1996regression}:
\begin{itemize}
\item Variable Saturation: when $p>n$, the Lasso selects at most $n$ variables after saturation \cite{efron2004least}
\item Non-Grouping Effect: suppose there is a group of highly correlated variables, the Lasso tends to select only one variable from the group. \cite{zou2005regularization}
\item Invalidation: when $p<n$ and there are high correlations between variables, it can be observed that the selection performance of the Lasso is determined by the ridge regression \cite{tibshirani1996regression}.
\end{itemize}
\par
Elastic net \cite{zou2005regularization} is proposed to address the above shortcomings by combining the Lasso and Ridge constraints. However, similar to the Lasso, its robustness can not be guaranteed. In order to address such shortcomings and simultaneously achieve the robustness, we propose the robust elastic net (REN) for sparse linear regression. Specifically, REN adopts the trimmed inner product to robustly estimate the covariance matrix while adopting a mixed $l_1$ and $l_2$ constraints. REN combines the robust Lasso \cite{chen2013robust} and the robust soft thresholding (RST) into a unified framework, and provides a unified performance guarantee. Our main theorem provides bounds on the fraction of corrupted samples that REN can tolerate, while still bounding the $l_2/l_1$ error of the support recovery.
\subsection{Related Work}
Sparse recovery without corruption has been well studied in the past years, even for the high-dimensional settings $p\geq n$. When the covariance matrix satisfies some conditions, e.g. Eigenvalue property \cite{bickel2009simultaneous,van2009conditions}, $\beta^*$ can be recovered with high probability. It is also learned that some random designs of matrix $\bm{X}$ also satisfies such property \cite{raskutti2010restricted}. Various estimators are also proposed to solve the $l_1$ regularized least squares problem, including basis pursuit \cite{chen1998atomic}, orthogonal matching pursuit \cite{tropp2004greed}, etc. However, these existing methods are not robust to outlier. Lasso and OMP are sensitive to corruptions. Even one entry in $\bm{X}$ or $\bm{y}$ may paralyze Lasso and OMP. Some work considers a natural modification of Lasso by adding a corruption term and regularizing it to account for certain types of corruptions. This modification is non-convex due to the bilinearity and therefore the performance guarantees can not be provided. In order to give provable performance guarantees, \cite{loh2012high} considers several corruptions including additive noise, multiplicative noise, missing data, and further construct the robust covariance matrix surrogates that could recover the sparse support with high probability. However, it has strong assumption on the corruption type. \cite{chen2013robust} improves the robust surrogates with trimmed estimators for Lasso. Similarly, \cite{feng2014robust} applies the trimmed statistics to proposes a robust logistic regression. Such idea of trimming also appears in designing robust PCA \cite{hauberg2014grassmann}. On the other hand, current work on robustifing the elastic net includes reducing the effect of outliers with reweighted procedure \cite{wang2014robust}, which is similar to the idea in \cite{yang2011robust}. Like most outlier detection based approaches, using reweighted procedure in the elastic net may fail with large-scale high-dimensional data.
\subsection{Contributions}
We utilize the trimmed statistic to design a robust elastic net for high-dimensional linear regression. The contributions can be summarized as follows.
\begin{itemize}
\item The robust elastic net regression is proposed with an efficient solving algorithm. The REN model extends the results in \cite{chen2013robust}.
\item We derives the statistical error bound for the support recovery of the REN model. As special cases, we also discuss the statistical error bound for the robust Lasso and the RST.
\item To validate the statistical error bound we derived, we further propose an efficient projected gradient descent and gives the provable optimization error bound when using this algorithm. The optimization error bound guarantees the statistical error bound can be achieved in practice.
\end{itemize}
\section{Robust Elastic Net Regression}
\subsection{Problem Setup}
We consider the sparse linear regression problem in the high-dimensional settings where the variable number is much larger than the number of observations, i.e. $\thickmuskip=2mu p\gg n$. Let $\beta^*$ be the $k$-sparse ($\thickmuskip=2mu k<p$) groundtruth parameter for the regressor, and $\{(\bm{x}_i,y_i)\}_{i=1}^{n+n_o}\in\mathbb{R}^p\times\mathbb{R}$ denote the covariate-response observation pairs. Without corruption, we assume these observations follow the linear model:
\begin{equation}\label{obslm}
y_i=\langle\bm{x}_i,\beta^*\rangle+\epsilon_i
\end{equation}
where $\bm{\epsilon}_i$ is additive noise that is independent of $\bm{x}_i$. The actual observation are corrupted by the following model.
\begin{mydef}[Fractional Adversarial Corruption]
Up to $n_o$ out of these $n+n_o$ pairs, including both $\bm{x}_i$ and $y_i$, are arbitrarily corrupted. The outlier fraction is bounded by $\frac{n_o}{n}$.
\end{mydef}
\par
Note that we make no assumption on these outliers. The outliers are even not necessarily formed into pairs, namely the corrupted rows in $\bm{X}$ and $\bm{y}$ can be different. They can follow any types of distributions. The corruption settings are also the same as \cite{chen2013robust}. Since the outlier fraction is bounded by $\frac{n_o}{n}$, we assume the outlier fraction is exactly $\frac{n_o}{n}$ without loss of generality. The remaining $n$ non-corrupted samples are called authentic for conciseness, and assumed to obey the standard sub-Gaussian construction \cite{loh2012high}.
\begin{mydef}[Sub-Gaussian Construction]
A random variable $v$ is sub-gaussian with parameter $\sigma$ if $\mathbb{E}\{\exp(tv)\}\leq\exp(\frac{t^2\sigma^2}{2})$ for all real $t$. A random matrix $\bm{X}\in\mathbb{R}^{n\times p}$ is sub-Gaussian with parameter $(\frac{1}{n}\Sigma_x,\frac{1}{n}\sigma^2_x)$ if: (a) each row $\bm{x}_i^T\in\mathbb{R}^p$ is sampled independently from a zero-mean distribution with covariance $\frac{1}{n}\Sigma_x$, (b) for any unit vector $\bm{u}\in\mathbb{R}_p$, the random variable $\bm{u}^T\bm{x}_i$ is sub-Gaussian with parameter at most $\sigma_x$.
\end{mydef}
\par
We assume the data matrix $\bm{X}$ is sub-Gaussian with parameter $(\frac{1}{n}\Sigma_x,\frac{1}{n}\sigma_x^2)$ before adversarial corruption, and the additive noise $\bm{\epsilon}$ is sub-Gaussian with parameter $\frac{1}{n}\sigma_{\epsilon}^2$.
\subsection{The REN Model}
This section presents the details of the proposed robust elastic net model. We use the trimmed statistics to robustify the elastic net. Specifically, we replace the inner product in the original elastic net regression with the trimmed inner product. The simple intuition behind the trimmed statistics is that entries with overly large magnitude in the inner product are likely to be corrupted. If the outliers have too large magnitude, they are not likely to have positive influence on the correlation and thus will be detrimental to the elastic net. Otherwise, they have bounded negative influence on the support recovery. We first consider a simple program
\begin{equation}\label{Mestimator}
\hat{\beta}\in \arg\min_{\|\beta\|_1\leq R}\big{\{}\frac{1}{2}\beta^T\hat{\Gamma}\beta-\langle\hat{\gamma},\beta\rangle\big{\}}
\end{equation}
where $\hat{\Gamma}, \hat{\gamma}$ denote the unbiased estimates of $\Sigma_x$ and $\Sigma_x\beta^*$ respectively. $\beta^*$ is the unique solution of this convex program given $R\geq\|\beta^*\|_1$. Lasso is also a special case of Eq. \eqref{Mestimator} with $\hat{\Gamma}_{\textnormal{LAS}}=\bm{X}^T\bm{X}$ and $\hat{\gamma}_{\textnormal{LAS}}=\bm{X}^T \bm{y}$. Under a suitable $\lambda$ and a positive semi-definite $\hat{\Gamma}$, Eq. \eqref{Mestimator} has an equivalent regularized form as
\begin{equation}\label{MestimatorR}
\hat{\beta}\in \arg\min_{\|\beta\|_1\in\mathbb{R}^p}\big{\{}\frac{1}{2}\beta^T \hat{\Gamma} \beta - \langle\hat{\gamma},\beta\rangle+\lambda\|\beta\|_1 \big{\}}
\end{equation}
where $\|\beta\|_1$ should be upper bounded by $b\sqrt{k}$ for a suitable constant $b$ if $\hat{\Gamma}$ has negative eigenvalues (in the case of $\hat{\Gamma}_{\textnormal{LAS}}, \hat{\gamma}_{\textnormal{LAS}}$). Since these two forms are equivalent, we merely study the first form. If $\thickmuskip=2mu\medmuskip=2mu \hat{\Gamma}_{\textnormal{EN}}=\alpha\bm{X}^T\bm{X}+(1-\alpha)\cdot\bm{I}, \alpha\in[0,1]$ and $\hat{\gamma}_{\textnormal{EN}}=\bm{X}^T \bm{y}$, Eq. \eqref{Mestimator} becomes the elastic net model. However, both the Lasso and the elastic net are fragile when corruption exists, especially under our noise settings. Thus we use the following trimmed statistic as robust surrogates for $\hat{\Gamma}$ and $\hat{\gamma}$:
\begin{equation}\label{Tstat}
\begin{aligned}
&\{\hat{\Gamma}_{\textnormal{REN}}\}_{ij}=\alpha\sum_{k=1}^{n_o} [\bm{X}_i,\bm{X}_j]_{(k)}+(1-\alpha)\cdot\bm{I}_{ij}\\
&\ \ \ \ \ \ \ \ \ \ \ \ \ \ \ \ \ \ \{\hat{\gamma}_{\textnormal{REN}}\}_{j}=\sum_{k=1}^{n_o} [\bm{X}_j,\bm{y}]_{(k)}
\end{aligned}
\end{equation}
where $\alpha\in[0,1]$, and $[\bm{u},\bm{v}]_{(k)}$ ($\bm{u},\bm{v}\in\mathbb{R}^h$) denotes the $k$th smallest variable in $\{\bm{q}_i=|\bm{u}_i\cdot\bm{v}_i|, \forall\ i\}$ such that $[\bm{u},\bm{v}]_{(1)}\leq[\bm{u},\bm{v}]_{(2)}\leq\cdots\leq[\bm{u},\bm{v}]_{(h)}$. $\bm{I}$ is an identity matrix. $\bm{X}_i$ denotes the $i$th column of the matrix $\bm{X}$. For example, $\sum_{k=1}^{n_o} [\bm{X}_i,\bm{X}_j]_{n_o}$ is to select the $n_o$ smallest element products between $|\bm{X}_i|$ and $|\bm{X}_j|$ and sum them. The robust surrogates are equivalent to adopting the summation of top $n_o$ variables in the element-wise vector multiplication. After substituting the robust surrogates, the optimization in REN may become non-convex because $\hat{\Gamma}$ may contain negative eigenvalues. Therefore it is also important to develop an algorithm that can approximate the optimum, otherwise the performance bound may be useless. We use the following polynomial-time projected gradient descent update:
\begin{equation}\label{Update}
\beta^{t+1}=\Pi_{l_1(R)}\big(\beta^t-\frac{1}{\eta}(\hat{\Gamma}_{\textnormal{REN}}\beta^t-\hat{\gamma}_{\textnormal{REN}})\big)
\end{equation}
where $\Pi_{l_1(R)}(\bm{v})=\arg\min_{\bm{z}}\{\|\bm{v}-\bm{z}\|_2\ |\ \|\bm{z}\|_1\leq R\}$ denotes Euclidean projection onto a $l_1$ ball of radius $R$. The optimization error is also bounded (the bound is given in the paper) so that the projected gradient descent can approximate the optimum with satisfactory accuracy. Some remarks for the REN model are in order:
\begin{itemize}
\item The REN model contains a family of sparse regression methods. The model parameter $\alpha$ controls the tradeoff between the ridge penalty and the robust trimmed covariance matrix (robust Lasso penalty). The REN model naturally bridges these two robustified model.
\item When $\alpha=0$, the REN model becomes the robust thresholding regression. When $\alpha=1$, the REN degenerates to the robust Lasso \cite{chen2013robust}.
\item The parameter $\alpha$ shrinks the sample covariance matrix towards the identity matrix. Adopting $\hat{\Gamma}_{\textnormal{REN}}$ is mathematically equivalent to replacing the trimmed covariance matrix $\Sigma_x$ with a shrunken version in the robust Lasso.
\end{itemize}
\subsection{Projected Gradient Descent for REN}
We present the proposed solving algorithm for the REN model. In order to validate the performance bound of the REN model, it is necessary to consider an algorithm whose solution should be close enough to the global optimum. The gradient of the quadratic loss function takes the form:
\begin{equation}\label{LossFunc}
\nabla\mathcal{L}(\beta)=\hat{\Gamma}_{\textnormal{REN}}\beta-\hat{\gamma}_{\textnormal{REN}}.
\end{equation}
We use the projected gradient descent that generates a sequence of iterates $\{\beta_t,t=0,1,2,\cdots\}$ by the recursion:
\begin{equation}\label{Update1}
\beta^{t+1}=\arg\min_{\|\beta\|\leq R}\big{\{}\mathcal{L}(\beta^t)+\langle\nabla\mathcal{L}(\beta^t),\beta-\beta^t\rangle+\frac{\eta}{2}\|\beta-\beta^t\|_2^2\big{\}}
\end{equation}
where $\eta>0$ is a step-size parameter. This update can be translated to $l_2$ projection onto the $l_1$ ball with radius $R$, which is in fact equivalent to Eq. \eqref{Update}. According to \cite{duchi2008efficient}, this update can be computed rapidly in $\mathcal{O}(p)$ time. For the regularized version Eq. \eqref{MestimatorR}, we just need to include $\lambda\|\beta\|_1$ in the update and perform two projections onto the $l_1$ ball \cite{agarwal2012fast}.
\par
When the objective functions in Eq. \eqref{Mestimator} and Eq. \eqref{MestimatorR} are convex, or equivalently, $\hat{\Gamma}$ is positive semi-definite, the update in Eq. \eqref{Update1} is guaranteed to converge to the global optimum. Although the $\hat{\Gamma}_{\textnormal{REN}}$ is not positive semi-definite, we still have the optimization bound that guarantees the obtained local optimum is close enough to the global optimum. The algorithm is summarized as follows.
{
\begin{algorithm}[h]
\small
    \caption{Robust Elastic Net}
    \label{alg1}
    \begin{algorithmic}[1]
    \REQUIRE~$\bm{X},\bm{y},R,n_o,\alpha$
    \ENSURE~$\hat{\beta}$\\
    \STATE Compute $\hat{\Gamma}_{\textnormal{REN}}$ and $\hat{\gamma}_{\textnormal{REN}}$ via Eq. \eqref{Tstat}.
    \WHILE {not reach maximal iteration or satisfied accuracy}
    \STATE Perform the projected gradient descent algorithm by iterating Eq. \eqref{Update1} to solve Eq. \eqref{Mestimator}.
    \ENDWHILE
    \STATE Output the final $\hat{\beta}$.
    \end{algorithmic}
 \end{algorithm}
 }
\vspace{-5mm}
\section{Performance Guarantees}
This section provides the performance guarantees for the REN model and its solving algorithm. The statistical error measures the upper bound of the $l_1/l_2$ difference between estimated $\hat{\beta}$ and the groudtruth $\beta^*$. The optimization error measures the upper bound of the $l_1/l_2$ difference between the $\beta^t$ (after $\medmuskip=0mu t-1$ iterations) and the global optimum $\hat{\beta}$. All the detail proofs are provided in the supplementary material.
\subsection{Preliminaries}
We first introduces a useful type of restricted eigenvalue (RE) conditions \cite{bickel2009simultaneous,van2009conditions}.
\begin{mydef}[Lower-RE Condition]
The matix $\hat{\Gamma}$ satisfies a lower restricted eigenvalue condition with curvature $\mu_1>0$ and tolerance $\tau(n,p)>0$ if for all $\theta\in\mathbb{R}^p$, the following condition holds:
\begin{equation}\label{LRE}
\theta^T\hat{\Gamma}\theta\geq\mu_1\|\theta\|^2_2-\tau(n,p)\|\theta\|_1^2
\end{equation}
\end{mydef}
\begin{mydef}[Upper-RE Condition]
The matix $\hat{\Gamma}$ satisfies a lower restricted eigenvalue condition with curvature $\mu_2>0$ and tolerance $\tau(n,p)>0$ if for all $\theta\in\mathbb{R}^p$, the following condition holds:
\begin{equation}\label{LRE}
\theta^T\hat{\Gamma}\theta\leq\mu_2\|\theta\|^2_2-\tau(n,p)\|\theta\|_1^2
\end{equation}
\end{mydef}
When the Lasso matrix $\hat{\Gamma}_{\textnormal{LAS}}$ satisfies both the lower-RE and upper-RE conditions, the Lasso guarantees low $l_2$ error for $\beta^*$ supported on any subset of size at most ${1}/{\tau_l(n,p)}$. In fact, we also prove that the proposed $\hat{\Gamma}_{\textnormal{REN}}$ also satisfies the lower-RE and upper-RE conditions.
\subsection{Statistical Error Bound}
\begin{thm}
Under the sub-Gaussian construction model, If we choose $R=\|\beta^*\|_1$ and the following conditions are satisfied:
\begin{align}\label{cond1}
&n\gtrsim \frac{\sigma^4_x}{\lambda^2_{min}(\Sigma_x)}k\log p,\\\label{cond2}
&\ \ \ \ \ \frac{n_o}{n}\lesssim \frac{\lambda_{min}(\Sigma_x)}{\sigma_x^2 k \log p},\\\label{cond3}
&\ \ \max_j\big{\{}\|{\Sigma_x}_j\|_0\big{\}}\leq c_1,
\end{align}
where $c_1$ is an absolute constant and ${\Sigma_x}_j$ denotes the $j$th column of the covariance matrix $\Sigma_x$, then with probability higher than $1-p^{-2}$, the output of the REN satisfy the following $l_1/l_2$ statistical error bound:
\begin{equation}\label{bound1}
\begin{aligned}
&\frac{1}{2\sqrt{k}}\|\hat{\beta}-\beta^*\|_1\leq\|\hat{\beta}-\beta^*\|_2\\
&\lesssim\frac{32}{\alpha\lambda_{min}(\Sigma_x)+ 4(1-\alpha)}\bigg(\frac{\alpha kn_o\log p}{n}\sigma_x^2\|\beta^*\|_2\\
&+\sqrt{\frac{k\sigma_\epsilon^2\log p}{n}}+\frac{n_o\log p \sqrt{k}\sigma_x}{n}\sqrt{\sigma_\epsilon^2+\sigma_x^2\|\beta^*\|_2}\\
&+(1-\alpha) c_2 k\sqrt{\frac{\sigma_x^2\log p}{n}}\|\beta^*\|_2^2\bigg)
\end{aligned}
\end{equation}
where $\alpha$ is the REN parameters and $c_2$ is an absolute constant.
\end{thm}
\begin{rem}
The number of non-zero entries in all columns of the covariance matrix $\Sigma_x$ is assumed to be bounded by a constant such that the statistical error bound for the REN model is not trivial. In other words, for every variable, the number of the other correlated variables is bounded.
\end{rem}
\begin{col}
Suppose that, under the sub-Gaussian construction model, Eq. \eqref{cond1}, Eq. \eqref{cond2} are both satisfied, $\alpha$ is set to 1, and $R$ is set to $\|\beta^*\|_1$, the output of the REN model is, with probability higher than $1-p^{-2}$, bounded by
\begin{equation}\label{bound2}
\begin{aligned}
&\frac{1}{2\sqrt{k}}\|\hat{\beta}-\beta^*\|_1\leq\|\hat{\beta}-\beta^*\|_2\\
&\lesssim\frac{32}{\lambda_{min}(\Sigma_x)}\bigg(\frac{kn_o\log p}{n}\sigma_x^2\|\beta^*\|_2+\sqrt{\frac{k\sigma_\epsilon^2\log p}{n}}\\
&+\frac{n_o\log p \sqrt{k}\sigma_x}{n}\sqrt{\sigma_\epsilon^2+\sigma_x^2\|\beta^*\|_2^2}\bigg)
\end{aligned}.
\end{equation}
\end{col}
\begin{rem}
When $\alpha=1$, the REN model becomes the robust Lasso. Note that the robust Lasso do not require Eq. \eqref{cond3} to be satisfied. Therefore, the error bound for the REN model ($\alpha=1$) is non-trivial without any assumptions on the covariance matrix $\Sigma_x$. The error bound for the robust Lasso coincides with \cite{chen2013robust}.
\end{rem}
\begin{col}
Suppose that, under the sub-Gaussian construction model, Eq. \eqref{cond1}, Eq. \eqref{cond2} are both satisfied, $\Sigma_x$ is assumed as $\bm{I}$, $\alpha$ is set to 0, and $R$ is set to $\|\beta^*\|_1$, the output of the REN model is, with probability higher than $1-p^{-2}$, bounded by
\begin{equation}\label{bound3}
\begin{aligned}
&\frac{1}{2\sqrt{k}}\|\hat{\beta}-\beta^*\|_1\leq\|\hat{\beta}-\beta^*\|_2\\
&\lesssim8\bigg(\sqrt{\frac{k\sigma_\epsilon^2\log p}{n}}+\frac{n_o\log p \sqrt{k}}{n}\sqrt{1+\sigma_\epsilon\|\beta^*\|_2^2}\\
&+k\sqrt{\frac{\log p}{n}}\|\beta^*\|_2\bigg)
\end{aligned}.
\end{equation}
\end{col}
\begin{rem}
When $\alpha=0$, the REN model becomes the robust soft thresholding. Note that the statistical error bound for the RST requires $\Sigma_x=\bm{I}$ to be satisfied. The RST is essentially a $l_1$ relaxation of robust thresholding regression \cite{chen2013robust}.
\end{rem}
\subsection{Optimization Error Bound}
\begin{lemma}
Under the sub-Gaussian construction model, $\hat{\Gamma}_{\textnormal{REN}}$ satisfies the lower-RE and upper-RE conditions with the following parameters:
\begin{align}\label{lem1con}
&\mu_1=\alpha\frac{\lambda_{\min}(\Sigma_x)}{2}+(1-\alpha),\\ &\mu_2=\alpha\frac{3\lambda_{\max}(\Sigma_x)}{2}+(1-\alpha),\\
&\tau(n,p)\leq\frac{\alpha}{8}\lambda_{\min}(\Sigma),
\end{align}
with high probability.
\end{lemma}
\begin{thm}
Under the conditions of \textbf{Theorem 1} and \textbf{Lemma 1}, for the projected gradient descend method, if we choose the step size as $2\mu_2$, then there must exist absolute constants $c_1,c_2>0$, $\gamma\in(0,1)$ such that the following inequalities hold for all $t\in\mathbb{N}$ with high probability,
\begin{equation}\label{bound4}
\begin{aligned}
\|\beta^t-\hat{\beta}\|_2^2&\leq\gamma^t\|\beta^0-\hat{\beta}\|_2^2+c_1\frac{\log p}{n}\|\hat{\beta}-\beta^*\|_1^2\\
&+\|\hat{\beta}-\beta^*\|_2^2,
\end{aligned}
\end{equation}
\begin{equation}\label{bound5}
\begin{aligned}
\|\beta^t-\hat{\beta}\|_1&\leq 2\sqrt{k}\|\beta^t-\hat{\beta}\|_2+2\sqrt{k}\|\hat{\beta}-\beta^*\|_2\\
&+2\|\hat{\beta}-\beta^*\|_1,
\end{aligned}
\end{equation}
where $\beta^t$ denotes the $t^{\textnormal{th}}$ gradient descent iterate.
\end{thm}
\begin{rem}
The optimization error bound controls the $l_1/l_2$-distance between the iterate $\beta^t$ and the global optimum $\hat{\beta}$ of Eq. \eqref{Mestimator}. The optimization error bound means for a large enough iteration $t$, the statistical $l_1/l_2$ error bound can be achieved. Note that $\beta^t$ can be computed in polynomial-time while $\hat{\beta}$ is difficult to obtain. Combining the statistical error bound and the optimization error bound gives the conclusion that the $l_1$ and $l_2$ error are bounded as $\mathcal{O}(\frac{k^{3/2}\log p}{n})$ and $\mathcal{O}(\frac{k\log p}{n})$ respectively.
\end{rem}
\section{Experiments}
\begin{figure}[!t]
\centering
  \renewcommand{\captionlabelfont}{\footnotesize}
  \setlength{\abovecaptionskip}{4pt}
  \setlength{\belowcaptionskip}{-5pt}
\includegraphics[width=3.5in]{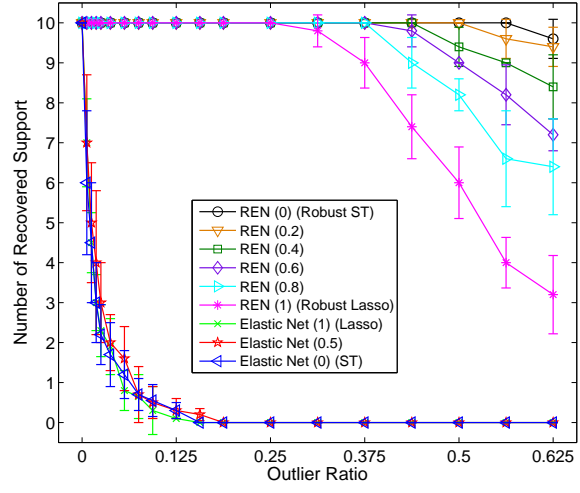}
\caption{. \footnotesize Number of recovered support for different methods while $\bm{X}$ has independent columns. The number beside REN is $\alpha$.}\label{fig1}
\end{figure}
\begin{figure}[!t]
\centering
  \renewcommand{\captionlabelfont}{\footnotesize}
  \setlength{\abovecaptionskip}{4pt}
  \setlength{\belowcaptionskip}{-5pt}
\includegraphics[width=3.5in]{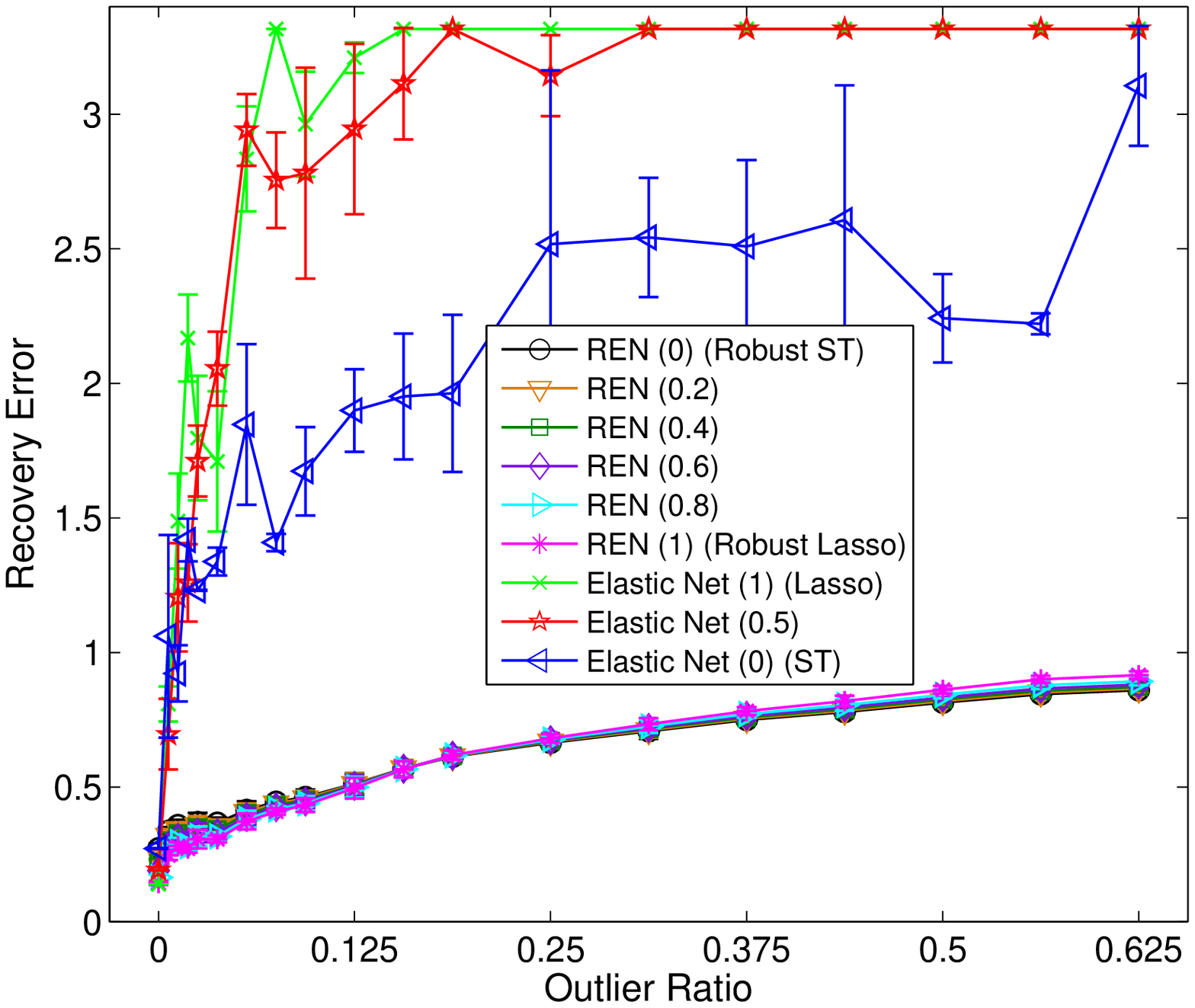}\caption{. \footnotesize $l_2$ support recovery error for different methods with independent columns of $\bm{X}$.}\label{fig2}
\end{figure}
\begin{figure}[!t]
\centering
  \renewcommand{\captionlabelfont}{\footnotesize}
  \setlength{\abovecaptionskip}{4pt}
  \setlength{\belowcaptionskip}{-5pt}
\includegraphics[width=3.5in]{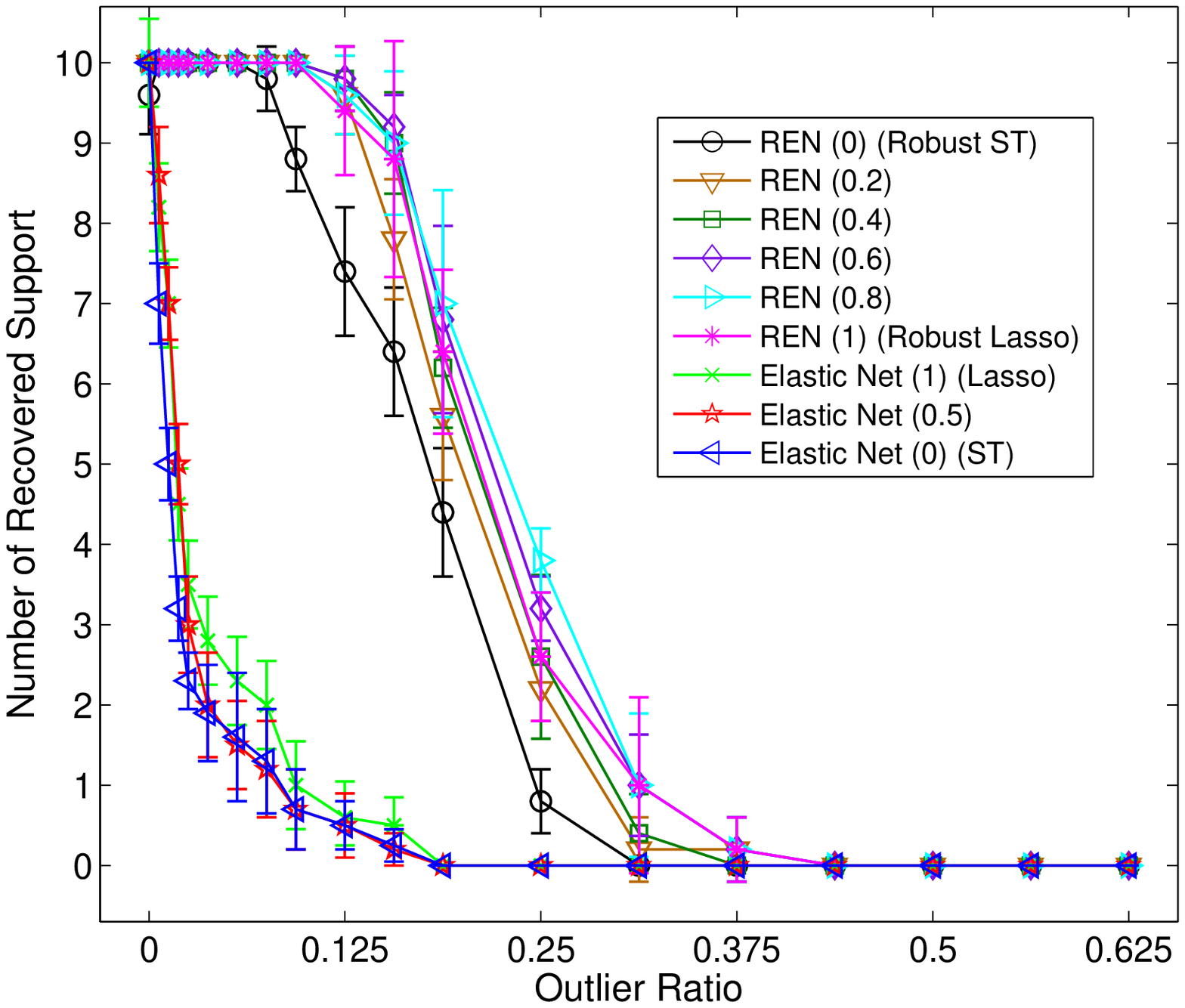}\caption{. \footnotesize Number of recovered support for different methods while $\bm{X}$ has dependent columns.}\label{fig3}
\end{figure}
\begin{figure}[!t]
\centering
  \renewcommand{\captionlabelfont}{\footnotesize}
  \setlength{\abovecaptionskip}{4pt}
  \setlength{\belowcaptionskip}{-5pt}
\includegraphics[width=3.5in]{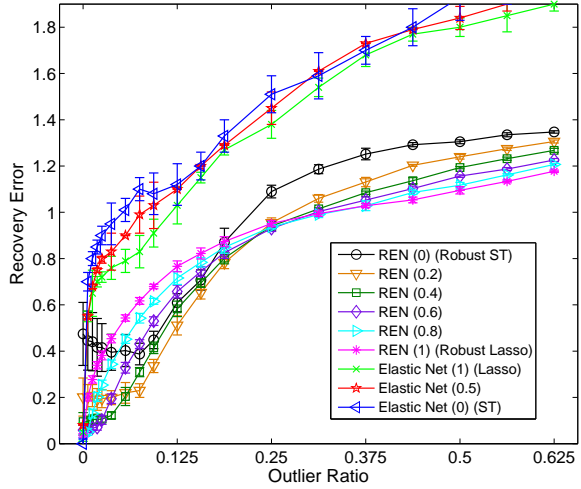}\caption{. \footnotesize $l_2$ support recovery error for different methods with dependent columns of $\bm{X}$.}\label{fig4}
\end{figure}
\subsection{Experimental Setup}
Under the the case where $\bm{X}$ has independent columns, we generate the authentic rows $\{\bm{X}^{\mathcal{A}},\bm{y}^{\mathcal{A}}\}$ following \cite{chen2013robust} via sub-Gaussian construction model with $\Sigma_x=\bm{I}$, $p=4000$, $n=1600$, $k=10$ and $\sigma_{\epsilon}=2$. The non-zero entries in $\beta^*$ are randomly selected as $\pm 1$. The corrupted rows $\{\bm{X}^{\mathcal{O}},\bm{y}^{\mathcal{O}}\}$ are generated via the following procedure. We first obtain $\theta^*$ by $\arg\min_{\|\theta\|_1\leq\|\beta^*\|_1,\theta\in\mathbb{R}^{p-k}}\|,\bm{y}^{\mathcal{A}}-\bm{X}^{\mathcal{A}}_{(\Lambda^*)^c}\theta\|_2$. Then we set $\bm{y}^{\mathcal{O}}=-\bm{X}^{\mathcal{O}}_{\Lambda^*}\beta^*$ where $\bm{X}^{\mathcal{O}}_{\Lambda^*}=\frac{3}{\sqrt{n}}\bm{S}$ and $\bm{S}$ is a $n_o\times k$ random matrix with each entry being $\pm 1$. Finally we let $\{\bm{X}^{\mathcal{O}}_{(\Lambda^*)^c}\}_i=(\frac{\bm{y}^{\mathcal{O}}_i}{\bm{B}^T_i\theta^*})\bm{B}^T_i$ where $\bm{B}_i$ is a $(n-k)$-dimensional vector with i.i.d. standard Gaussian distributed entries. While $\bm{X}$ has correlated columns, the data is generated using $\sigma_e=1$, $\Sigma_x$ with diagonal being 1 and the other entries being 0.4. The other parameters follows the independent case. For more detailed settings refer to the corresponding sections.
\subsection{Robustness Evaluation}
The section evaluates the robustness of the REN model under different $\alpha$. We use the number of recovered support and the $l_2$ recovery error (${\|\hat{\beta}-\beta^*\|_2}/{\|\beta^*\|_2}$) for evaluation.  We first consider the case where the data matrix $\bm{X}$ has independent columns. We vary the REN parameter $\alpha$ from 0 to 1 with step 0.2. The number of recovered support is shown in Fig. \ref{fig1} and the corresponding $l_2$ support recovery error is given in Fig. \ref{fig2}. From Fig. \ref{fig1} and Fig. \ref{fig2}, one can observe that REN consistently outperforms the original elastic net (including both Lasso and soft thresholding) in terms of both number of recovered support and the $l_2$ recovery error. REN can perfectly recover all the support number when outlier ratio is less than 0.5. Under the independent case, we find that REN with $\alpha=0$, i.e. RST, performs best and REN with $\alpha=1$, i.e. robust Lasso, has the worse performance. With $\alpha$ decreasing, the performance of REN gets better. It is partially because the authentic data satisfies $\Sigma_x=\bm{I}$ that is also the assumption for the statistical error bound of RST. The bound for RST is actually smaller than the bound for robust Lasso, so the performance of RST is the best when $\Sigma_x=\bm{I}$.
\begin{figure}[t!]
\centering
  \renewcommand{\captionlabelfont}{\footnotesize}
  \setlength{\abovecaptionskip}{4pt}
  \setlength{\belowcaptionskip}{-5pt}
\includegraphics[width=3.5in]{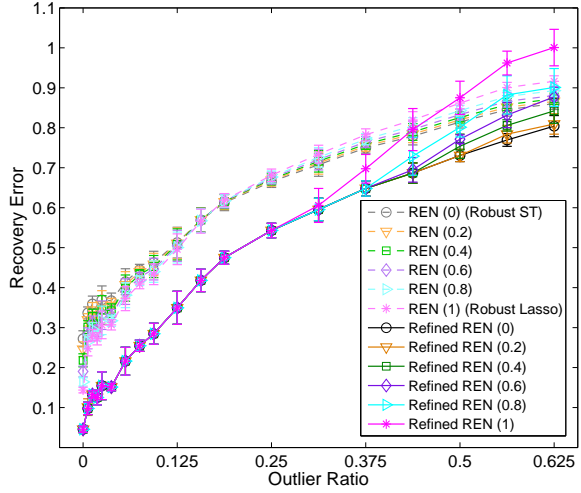}\caption{. \footnotesize Number of recovered support for different methods while $\bm{X}$ has dependent columns.}\label{fig5}
\end{figure}
\begin{figure}[t!]
\centering
  \renewcommand{\captionlabelfont}{\footnotesize}
  \setlength{\abovecaptionskip}{4pt}
  \setlength{\belowcaptionskip}{-5pt}
\includegraphics[width=3.5in]{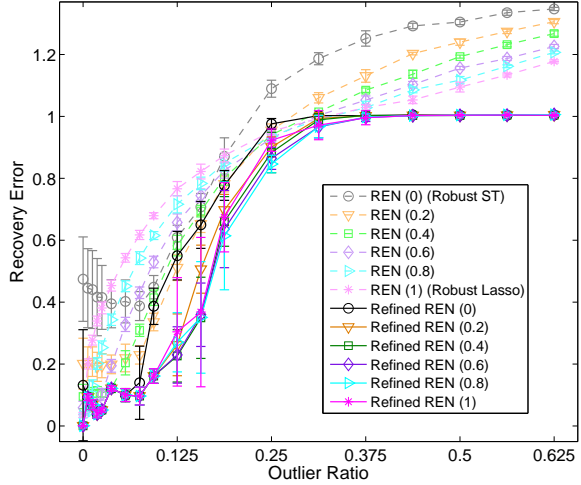}\caption{. \footnotesize $l_2$ support recovery error for different methods with dependent columns of $\bm{X}$.}\label{fig6}
\end{figure}
\par
Then we evaluate the performance of REN in the case that $\bm{X}$ has correlated columns. The results are shown in Fig. \ref{fig3} and Fig. \ref{fig4}. One can observe that REN still consistently outperforms the original elastic net (including both Lasso and soft thresholding). Although the performance of REN decreases with data containing correlated columns, it can still perfectly recover the locations of support with up to 90 outliers. The corresponding $l_2$ recovery error is also much smaller than the original elastic net. It can be observed that the elastic net with $\alpha=1$ has the best performance against outliers, which coincides with the statistical error for the robust Lasso, because Robust Lasso do not requires the bounded number of non-zero values in the columns of $\Sigma_x$.
\subsection{Refinement for Recovery Error}
We observe that even REN recovers most locations of support, but the $l_2$ corresponding recovery error is still very large. (sometime even larger than methods that do not recover the same number of support.) Thus we consider a standard refinement strategy for the REN model. When we recover the locations of support, we simply force the other locations in $\hat{\beta}$ to be zero ,and only use these recovered support to perform the least square regression with $\{\hat{\Gamma}_{\textnormal{REN}}\}_{S,S}$ and $\{\hat{\gamma}_{\textnormal{REN}}\}_S$. $S$ is the set of indexes corresponding to the location of the recovered support. Specifically, we use $\hat{\beta}_S=\{\hat{\Gamma}_{\textnormal{REN}}\}_{S,S}^{-1}\cdot\{\hat{\gamma}_{\textnormal{REN}}\}_S$ and set $\hat{\beta}_{S^c}$ as zeros. The refined $l_2$ recovery error for REN is shown in Fig. \ref{fig5} and Fig. \ref{fig6}. From Fig. \ref{fig5}, one can see that the refined recovery error is much smaller than the original error. However, the refined error gets worse rapidly from a specific outlier ratio. For example, the refined error for the REN ($\alpha=1$) gets worse rapidly from the outlier ratio 0.3125. This is because the support location recovery starts to get worse from the same point. From Fig. \ref{fig6}, one can still see the advantages of the refinement strategy but with less gain, and the results are not as stable as the independent case. The refined error stays to be 1 from outlier ratio 0.25, which coincides with the results in Fig. \ref{fig3} that the recovered support locations are close to 0 from outlier ratio 0.25. Being 1 in refined error indicates that nearly all support locations are wrongly estimated.
\section{Concluding Remarks}
We consider the high-dimensional sparse regression problem in this paper. By observing that the elastic net model (including Lasso and soft thresholding) fails with even a very small fraction of outliers, we proposed a robust elastic net model using the trimmed inner product as a robust counterpart to replace the original inner product. The intuition behind is that the trimmed inner product can not be significantly affected by a bounded number of arbitrarily corrupted points, i.e., arbitrary outliers. Such simple idea works extremely well in high-dimensional settings. We further give the performance guarantees (both statistical $l_{1/2}$ error bound and the optimization error bound) for the REN model. The statistical error shows the guarantees for the support recovery, while the optimization bound makes sure that we can use projected gradient descent to approximate the optima and achieve the statistical error bound. Experimental results match the theoretical performance guarantees nicely and also validate the robustness of the REN model.
\section{Proof of Theorem 1}
\subsection{Technical Lemmas}
In order to prove Theorem 1, we need two technical lemmas. The first lemma bounds the maximum of independent sub-Gaussian random variables. The second lemma is a standard concentration result for the sum of squares of independent sub-Gaussian random variables.
\begin{lemma}
Given that $Z_1,\cdots,Z_m$ are $m$ independent sub-Gaussian random variables with parameter $\sigma$, we have $\max_{i=1,\cdots,m}\|Z_i\|\leq 4\sigma\sqrt{\log m+\log p}$ with high probability.
\end{lemma}
\begin{proof}
Define $Z_m=\max_iZ_i$. According to the definition of sub-Gaussianity, we have
\begin{equation}
\begin{aligned}
\mathbb{E}\{e^{\frac{tZ_m}{\sigma}}\}&=\mathbb\{\max_ie^{\frac{tZ_i}{\sigma}}\}\\
&\leq \sum_i\mathbb{E}\{e^{\frac{tZ_i}{\sigma}}\}\\
&\leq me^{\frac{t^2}{2}}\\
&=e^{\frac{t^2}{2}+\log m}.
\end{aligned}
\end{equation}
Using Markov inequality, we obatin
\begin{equation}
\begin{aligned}
P(Z_m\geq\sigma t)&=P(e^{\frac{tZ_m}{\sigma}}\geq e^{t^2})\\
&\leq e^{-t^2}\mathbb{E}\{e^{\frac{tZ_m}{\sigma}}\}\\
&\leq e^{-t^2+t^2/2+\log m}\\
&=e^{-\frac{t^2}{2}+\log m}.
\end{aligned}
\end{equation}
By symmetry, we have
\begin{equation}
P(\min_i Z_i\leq -\sigma t)\leq e^{-\frac{t^2}{2}+\log m}.
\end{equation}
Therefore a union bound gives
\begin{equation}
\begin{aligned}
P(\max_i|Z_i|\geq \sigma t)&\leq P(\max_i Z_i\geq \sigma t) + P(\min_i Z_i \leq -\sigma t)\\
&\leq 2e^{-\frac{1}{2}t^2+\log m},\\
\end{aligned}
\end{equation}
which yields the result with $t=4\sqrt{\log m+\log p}$.
\end{proof}
\begin{lemma}
Let $Y_1,\cdots,Y_n$ be n i.i.d zero-mean sub-Gaussian random variables with parameter $\frac{1}{\sqrt{n}}$ and variance at most $\frac{1}{n}$. Then we have
\begin{equation}
|\sum_{i=1}^{n}Y_i^2-1|\leq c_1\sqrt{\frac{\log p}{n}}
\end{equation}
with high probability for some absolute constant $c_1$. If $Z_1,\cdots,Z_n$ are also i.i.d. zero-mean sub-Gaussian random variables with parameter $\frac{1}{\sqrt{n}}$ and variance at most $\frac{1}{n}$, and independent of $Y_1,\cdots,Y_n$, then we have
\begin{equation}
|\sum_{i=1}^{n}Y_iZ_i|\leq c_2\sqrt{\frac{\log p}{n}}
\end{equation}
with high probability for some absolute constant $c_2$.
\end{lemma}
\begin{proof}
The proof follows directly from Lemma 14 in the supplementary material of \cite{loh2012high}.
\end{proof}
\subsection{Main Proof}
Recall that $\hat{\Gamma}_{\textnormal{REN}}=\alpha \bm{X}^T_{\mathcal{T}}\bm{X}_{\mathcal{T}}+(1-\alpha)\bm{I}$ and $\hat{\gamma}_{\textnormal{REN}}=\bm{X}^T_{\mathcal{T}} \bm{X}_{\mathcal{T}}\beta^*$, where $\mathcal{T}$ is the index set corresponding to the trimmed inner product. Let $\Delta=\hat{\beta}-\beta^*$, $F=\hat{\Gamma}_{\textnormal{REN}}-\alpha\bm{X}_{\mathcal{A}}^T\bm{X}_{\mathcal{A}}-(1-\alpha)\bm{I}$, $f=\hat{\gamma}_{\textnormal{REN}}-\bm{X}_{\mathcal{A}}^T\bm{X}_{\mathcal{A}}\beta^*$, and $\mathcal{A}$ is the index set of authentic data, $\mathcal{O}$ is the index set of outliers, $S=\textnormal{support}(\beta^*)$. For any vector $b\in\mathbb{R}^p$, we define $b_S$ as the vector with $(b_S)_i=b_i$ for $i\in S$ and $(b_S)_i=0$ for $i\notin S$.
\par
With $\hat{\beta}$ satisfying the constraint in the optimization problem in Robust Lasso, we have
\begin{equation}
\begin{aligned}
\|\beta^*\|_1&\geq \|\beta^*+\Delta\|_1\\
&=\|\beta^*+\Delta_S\|_1+\|\Delta_{S^c}\|_1\\
&\geq \|\beta^*\|_1-\|\Delta_S\|_1+\|\Delta_{S^c}\|_1,
\end{aligned}
\end{equation}
which follows that $\|\Delta_{S^c}\|\leq\|\Delta_S\|_1$. Given $\|S\|=k$, we have the following inequality
\begin{equation}
\begin{aligned}
\|\Delta\|_1&=\|\Delta_S\|_1+\|\Delta_{S^c}\|_1\\
&\leq 2\|\Delta_S\|_1\\
&\leq 2\sqrt{k}\|\Delta_S\|_2\\
&\leq 2\sqrt{k} \|\Delta\|_2
\end{aligned}
\end{equation}
\par
Under the assumption for $n$ in this theorem, Lemma 1 in \cite{loh2012high} guarantees that the authentic $\bm{X}_{\mathcal{A}}$ satisfies the Restricted Strong Convexity (RSC) under the assumption of this theorem:
\begin{equation}\label{lem2eq1}
\bm{u}^T(\bm{X}_{\mathcal{A}}^T\bm{X}_{\mathcal{A}})\bm{u}\geq \frac{1}{4}\lambda_{\min}(\Sigma_x)\|\bm{u}\|_2^2,\ \ \ \ \forall \bm{u}: \|\bm{u}\|_1\leq 2\sqrt{k}\|\bm{u}\|_2
\end{equation}
Because $\|\Delta\|_1\leq2\sqrt{k}\|\Delta\|_2$, we obtain
\begin{equation}\label{eq2}
\begin{aligned}
\Delta^T\hat{\Gamma}_{\textnormal{REN}}^T\Delta &= \alpha\Delta^T(\bm{X}_{\mathcal{A}}^T\bm{X}_{\mathcal{A}})\Delta+\Delta^TF\Delta+(1-\alpha)\Delta^T\Delta\\
&\geq \frac{\alpha}{2}\lambda_{\min}(\Sigma_x)\|\Delta\|_2^2-\|F\|_{\infty}\sum_{i,j}|\Delta_i||\Delta_j| +(1-\alpha)\Delta^T\Delta\\
&\geq \frac{\alpha}{2}\lambda_{\min}(\Sigma_x)\|\Delta\|_2^2-4k\|F\|_{\infty}\|\Delta\|^2_2 +(1-\alpha)\|\Delta\|_2^2\\
\end{aligned}
\end{equation}
For $F$, we can bound it with
\begin{equation}
\begin{aligned}
F_{ij}&=\alpha\bigg(\langle \bm{X}_i,\bm{X}_j\rangle_{n_o}-\langle \{\bm{X}_{\mathcal{A}}\}_i,\{\bm{X}_{\mathcal{A}}\}_j\rangle\bigg)\\
&=-\alpha\sum_{k\in \mathcal{T}\cap\mathcal{A}}\bm{X}_{ki}\bm{X}_{kj}+\alpha\sum_{k\in \mathcal{T}^c\cap\mathcal{O}}\bm{X}_{ki}\bm{X}_{kj}\\
&\leq 2n_o\alpha\bigg( \max_{k\in\mathcal{A}}|\bm{X}_{ki}| \bigg)\bigg( \max_{k\in\mathcal{A}}|\bm{X}_{kj}| \bigg)
\end{aligned}
\end{equation}
Since $\bm{X}_{ki},k\in\mathcal{A}$ are independent sub-Gaussian variable with parameter $\frac{1}{n}\sigma_x^2$, Lemma 1 concludes $\max_{k\in\mathcal{A}}|\bm{X}_{ki}|\lesssim \sigma_x\sqrt{\frac{\log p}{n}}$ with high probability. Thus it follows from a union bound over $(i,j)$ that
\begin{equation}
\|F\|_{\infty}\lesssim\alpha\frac{n_o\log p}{n}\sigma_x^2.
\end{equation}
Recall the assumption of the theorem that $\frac{n_o}{n}\lesssim\frac{\lambda_{\min}(\Sigma_x)}{\sigma_x^2k\log p}$. Thus we have
\begin{equation}\label{eq3}
\|F\|_{\infty}\leq\alpha\frac{\lambda_{\min}(\Sigma_x)}{16k}
\end{equation}
Combining Eq. \eqref{eq2} and Eq. \eqref{eq3}, we have
\begin{equation}
\Delta^T\hat{\Gamma}^T_{\textnormal{REN}}\Delta\geq \frac{\alpha}{4}\lambda_{\min}(\Sigma_x)\|\Delta\|_2^2+(1-\alpha)\|\Delta\|_2^2
\end{equation}
According to Holder's inequality and $\|\Delta\|_1\leq2\sqrt{k}\|\Delta\|_2$, we obtain
\begin{equation}\label{eq4}
\begin{aligned}
\langle\hat{\gamma}_{\textnormal{REN}}-\hat{\Gamma}_{\textnormal{REN}}\beta^*,\Delta\rangle&\leq\|\hat{\gamma}_{\textnormal{REN}}-\hat{\Gamma}_{\textnormal{REN}}\beta^*\|_{\infty}\|\Delta\|_1\\
&\leq 4\sqrt{k}\|\hat{\gamma}_{\textnormal{REN}}-\hat{\Gamma}_{\textnormal{REN}}\beta^*\|_{\infty}\|\Delta\|_2
\end{aligned}
\end{equation}
in which we note that
\begin{equation}\label{eq5}
\begin{aligned}
\|\hat{\gamma}_{\textnormal{REN}}-\hat{\Gamma}_{\textnormal{REN}}\beta^*\|_{\infty}&\leq\|\bm{X}^T_{\mathcal{A}}\bm{X}_{\mathcal{A}}\beta^*-\hat{\Gamma}_{\textnormal{REN}}\beta^*\|_{\infty} +\|\hat{\gamma}_{\textnormal{REN}}-\bm{X}^T_{\mathcal{A}}\bm{X}_{\mathcal{A}}\beta^*\|_{\infty}\\
&=\|\alpha\bm{X}^T_{\mathcal{A}}\bm{X}_{\mathcal{A}}\beta^*+(1-\alpha)\bm{I}\beta^*-\hat{\Gamma}_{\textnormal{REN}}\beta^*+(1-\alpha)\bm{X}^T_{\mathcal{A}}\bm{X}_{\mathcal{A}}\beta^*-(1-\alpha)\bm{I}\beta^*\|_\infty\\
&\ \ \ \ \ \ + \|\hat{\gamma}_{\textnormal{REN}}-\bm{X}^T_{\mathcal{A}}\bm{X}_{\mathcal{A}}\beta^*\|_\infty\\
&\leq\|\alpha\bm{X}^T_{\mathcal{A}}\bm{X}_{\mathcal{A}}\beta^*+(1-\alpha)\bm{I}\beta^*-\hat{\Gamma}_{\textnormal{REN}}\beta^*\|_\infty + \|(1-\alpha)\bm{X}^T_{\mathcal{A}}\bm{X}_{\mathcal{A}}\beta^*-(1-\alpha)\bm{I}\beta^*\|_\infty\\
&\ \ \ \ \ \ + \|\hat{\gamma}_{\textnormal{REN}}-\bm{X}^T_{\mathcal{A}}\bm{X}_{\mathcal{A}}\beta^*\|_\infty\\
&=\|F\beta^*\|_\infty + (1-\alpha)\|\bm{X}^T_{\mathcal{A}}\bm{X}_{\mathcal{A}}\beta^*-\bm{I}\beta^*\|_\infty + \|f\|_\infty
\end{aligned}
\end{equation}
For the first term $\|F\beta^*\|_\infty$ in Eq. \eqref{eq5}, we can bound it using $\|F\|_{\infty}\lesssim\alpha\frac{n_o\log p}{n}\sigma_x^2$ and the $k$-sparsity of $\beta^*$:
\begin{equation}
\|F\beta^*\|_\infty\leq\frac{\alpha \sqrt{k}n_o\log p}{n}\sigma_x^2\|\beta^*\|_2
\end{equation}
For the second term $(1-\alpha)\|\bm{X}^T_{\mathcal{A}}\bm{X}_{\mathcal{A}}\beta^*-\bm{I}\beta^*\|_\infty$ in Eq. \eqref{eq5}, we have
\begin{equation}
(1-\alpha)\|\bm{X}^T_{\mathcal{A}}\bm{X}_{\mathcal{A}}\beta^*-\bm{I}\beta^*\|_\infty \leq(1-\alpha)\|\bm{X}_{\mathcal{A}}^T\bm{X}_{\mathcal{A}}-\bm{I}\|_\infty\cdot\sqrt{k}\|\beta^*\|_2.
\end{equation}
Using Lemma 2 and the theorem assumption that $\max_j \big{\{} \| \Sigma_{xj} \|_0 \big{\}} \leq c_3$ where $\Sigma_{xj}$ denotes the $j$th column of the covariance matrix $\Sigma_x$, we have
\begin{equation}
(1-\alpha)\|\bm{X}^T_{\mathcal{A}}\bm{X}_{\mathcal{A}}\beta^*-\bm{I}\beta^*\|_\infty \leq(1-\alpha)c_3 \sqrt{\frac{k\log p}{n}}\|\beta^*\|_2.
\end{equation}
For the third term $\|f\|_\infty$ in Eq. \eqref{eq5}, we decompose $f_j$ as
\begin{equation}
\begin{aligned}
f_j&=\langle\bm{X}_j,\bm{y}\rangle_{n_o}-\langle \bm{X}_{\mathcal{A}j}, \bm{X}_{\mathcal{A}}\beta^* \rangle\\
&=\sum_{i\in \mathcal{T}^c}\bm{X}_{ij}\bm{y}_i-\langle \bm{X}_{\mathcal{A}j}, \bm{X}_{\mathcal{A}}\beta^* \rangle\\
&=\bigg( \sum_{i\in\mathcal{A}}\bm{X}_{ij}\bm{y}_i - \langle \bm{X}_{\mathcal{A}j}, \bm{X}_{\mathcal{A}}\beta^* \rangle \bigg)-\sum_{i\in\mathcal{T}\cap\mathcal{A}}\bm{X}_{ij}\bm{y}_i +\sum_{i\in\mathcal{T}^c\cap\mathcal{O}}\bm{X}_{ij}\bm{y}_i\\
&=\langle\bm{X}_{\mathcal{A}j},\bm{\epsilon}\rangle-\sum_{i\in\mathcal{T}\cap\mathcal{A}}\bm{X}_{ij}\bm{y}_i +\sum_{i\in\mathcal{T}^c\cap\mathcal{O}}\bm{X}_{ij}\bm{y}_i
\end{aligned}
\end{equation}
From Lemma 2, we obtain $\langle\bm{X}_{\mathcal{A}j},\bm{\epsilon}\rangle\approx\sqrt{\frac{\sigma_{\epsilon}^2\log p}{n}}$ with high probability. Under  the sub-Gaussian construction model, each $\bm{y}_i,i\in\mathcal{A}$ is sub-Gaussian with parameter $\frac{\sigma_\epsilon^2+\sigma_x^2\|\beta^*\|_2^2}{n}$. Based on the Lemma 1, we have
\begin{equation}
\bigg{|}-\sum_{i\in\mathcal{T}\cap\mathcal{A}}\bm{X}_{ij}\bm{y}_i +\sum_{i\in\mathcal{T}^c\cap\mathcal{O}}\bm{X}_{ij}\bm{y}_i\bigg{|}\lesssim\frac{n_o\log p}{n}\sigma_x^2\sqrt{\sigma_\epsilon^2+\sigma_x^2\|\beta^*\|^2_2}
\end{equation}
which follows from a union bound and gives
\begin{equation}
\|f\|_\infty\lesssim\sqrt{\frac{\sigma_\epsilon^2 \log p}{n}}+\frac{n_o\log p}{n}\sigma_x^2\sqrt{\sigma_\epsilon^2+\sigma_x^2\|\beta^*\|_2^2}
\end{equation}
Combining Eq. \eqref{eq4}, Eq. \eqref{eq5} and the other pieces, we have
\begin{equation}
\begin{aligned}
\langle\hat{\gamma}_{\textnormal{REN}}-\hat{\Gamma}_{\textnormal{REN}}\beta^*,\Delta\rangle &\lesssim4\sqrt{k}\|\Delta\|_2\bigg( \frac{\alpha\sqrt{k}n_o\log p}{n}\sigma_x^2\|\beta\|_2 + \sqrt{\frac{\sigma_\epsilon^2 \log p}{n}}\\
&\ \ \ \ \ \ + \frac{n_o\log p}{n}\sigma_x\sqrt{\sigma_\epsilon^2+ \sigma_x^2\|\beta^*\|_2^2} + (1-\alpha)c_3\sqrt{k}\sqrt{\frac{\log p}{n}}\|\beta^*\|_2 \bigg)
\end{aligned}
\end{equation}
\par
According to the optimality of $\hat{\beta}$, we have
\begin{equation}
\frac{1}{2}\hat{\beta}^T\hat{\Gamma}_{\textnormal{REN}}\hat{\beta}-\hat{\gamma}_{\textnormal{REN}}^T\hat{\beta}\leq \frac{1}{2}\hat{\beta}^{*T}\hat{\Gamma}_{\textnormal{REN}}\hat{\beta^*}-\hat{\gamma}_{\textnormal{REN}}^T\hat{\beta^*}
\end{equation}
which can be rearranged as
\begin{equation}
\frac{1}{2}\Delta^T\hat{\Gamma}_{\textnormal{REN}}\Delta\leq\langle \hat{\gamma}_{\textnormal{REN}}-\hat{\Gamma}_{\textnormal{REN}}\beta^*,\Delta \rangle
\end{equation}
Combining all the derived inequalities, we have
\begin{equation}
\begin{aligned}
(\frac{\alpha}{8}\lambda_{\min}(\Sigma_x)+\frac{1-\alpha}{2})\|\Delta\|_2^2&\leq\frac{1}{2}\Delta^T\hat{\Gamma}_{\textnormal{REN}}\Delta\leq \langle \hat{\gamma}_{\textnormal{REN}}-\hat{\Gamma}_{\textnormal{REN}}\beta^*,\Delta \rangle\\
&\leq 4\sqrt{k}\| \hat{\gamma}_{\textnormal{REN}}-\hat{\Gamma}_{\textnormal{REN}}\beta^*,\Delta \|_{\infty}\|\Delta\|_2
\end{aligned}
\end{equation}
which derives
\begin{equation}
\begin{aligned}
\frac{1}{2\sqrt{k}}\|\Delta\|_1\leq\|\Delta\|_2&\leq \frac{4\sqrt{k}}{\frac{\alpha}{8}\lambda_{\min}(\Sigma_x)+\frac{1-\alpha}{2}}\|\hat{\gamma}_{\textnormal{REN}}-\hat{\Gamma}_{\textnormal{REN}}\beta^*\|_{\infty}\\
&\leq \frac{32}{\alpha\lambda_{\min}(\Sigma_x)+4(1-\alpha)}\bigg( \frac{\alpha kn_o\log p}{n}\sigma_x^2\|\beta\|_2 + \sqrt{\frac{k\sigma_\epsilon^2 \log p}{n}}\\
&\ \ \ \ \ \ + \frac{n_o\sqrt{k}\log p}{n}\sigma_x\sqrt{\sigma_\epsilon^2+ \sigma_x^2\|\beta^*\|_2^2} + (1-\alpha)c_3k\sqrt{\frac{\log p}{n}}\|\beta^*\|_2 \bigg)
\end{aligned}
\end{equation}
which finishes the proof of the theorem.
\section{Proof of Corollary 1}
In fact, the proof of Corollary 1 follows directly from Theorem 1. The difference is that if $\alpha=1$ in Eq. \eqref{eq5}, then the second term $(1-\alpha)\|\bm{X}^T_{\mathcal{A}}\bm{X}_{\mathcal{A}}\beta^*-\bm{I}\beta^*\|_\infty$ is removed. Therefore, we do not need to use the condition $\max_j \big{\{} \| \Sigma_{xj} \|_0 \big{\}} \leq c_3$ to bound this term. So the conclusion of Corollary 1 is the same as Theorem 1 with $\alpha=1$, but they differs in the required conditions.
\section{Proof of Corollary 2}
The proof of Corollary 2 is essentially the same as Theorem 1. Corollary 2 can be directly concluded with $\alpha=0$ and $\Sigma_x=\bm{I}$.
\section{Proof of Lemma 1}
From Eq. \eqref{lem2eq1}, Eq. \eqref{eq2} and Eq. \eqref{eq3} in the proof of Theorem 1 and, we can derive the following inequality:
\begin{equation}
\begin{aligned}
\Delta^T\hat{\Gamma}_{\textnormal{REN}}^T\Delta &= \alpha\Delta^T(\bm{X}_{\mathcal{A}}^T\bm{X}_{\mathcal{A}})\Delta+\Delta^TF\Delta+(1-\alpha)\Delta^T\Delta\\
&\geq \frac{\alpha}{2}\lambda_{\min}(\Sigma_x)\|\Delta\|_2^2-\|F\|_{\infty}\sum_{i,j}|\Delta_i||\Delta_j| +(1-\alpha)\Delta^T\Delta\\
&= \frac{\alpha}{2}\lambda_{\min}(\Sigma_x)\|\Delta\|_2^2-\|F\|_{\infty}\|\Delta\|^2_1 +(1-\alpha)\|\Delta\|_2^2\\
&= \bigg(\frac{\alpha}{2}\lambda_{\min}(\Sigma_x)+(1-\alpha)\bigg)\|\Delta\|_2^2-\|F\|_{\infty}\|\Delta\|^2_1\\
&\geq\bigg(\frac{\alpha}{2}\lambda_{\min}(\Sigma_x)+(1-\alpha)\bigg)\|\Delta\|_2^2-\frac{\alpha\lambda_{\min}(\Sigma_x)}{8}\|\Delta\|^2_1\\
\end{aligned}
\end{equation}
which satisfies the lower-RE condition.
\par
Similarly, according to Lemma 1, Lemma 13 in the supplementary material of \cite{loh2012high} and the restricted smoothness, we can obtain
\begin{equation}
\Delta^T\hat{\Gamma}_{\textnormal{REN}}^T\Delta \leq\bigg(\frac{3\alpha}{2}\lambda_{\max}(\Sigma_x)+(1-\alpha)\bigg)\|\Delta\|_2^2+\frac{\alpha\lambda_{\min}(\Sigma_x)}{8}\|\Delta\|^2_1\\
\end{equation}
which satisfies the upper-RE condition. Therefore, one possible set of values for $\mu_1$, $\mu_2$ and $\tau(n,p)$ is
\begin{align}\label{lem1con}
&\mu_1=\alpha\frac{\lambda_{\min}(\Sigma_x)}{2}+(1-\alpha),\\ &\mu_2=\alpha\frac{3\lambda_{\max}(\Sigma_x)}{2}+(1-\alpha),\\
&\tau(n,p)=\frac{\alpha}{8}\lambda_{\min}(\Sigma),
\end{align}
\par
Combing pieces, we finish the proof of Lemma 1 in the main paper. It can also be proved with a simple extension using the Appendix 5 in the supplementary material of \cite{chen2013robust}.
\section{Proof of Theorem 2}
Using the Lemma 1 in the main paper, we know that the $\hat{\Gamma}_{\textnormal{REN}}$ satisfies the lower-RE and upper-RE conditions with certain parameters. We can directly follow the Theorem 2 in \cite{loh2012high} to prove this theorem.
{
\bibliographystyle{plain}
\bibliography{REN}
}
\end{document}